%% file: main.tex
\documentclass{article}
\usepackage{times}
\usepackage{epsfig}
\usepackage{graphicx}
\usepackage{amsmath, amsthm, amssymb}
\usepackage{natbib}
\usepackage{url}

\input{symbol.tex}

\title{From Betting to Empirical Bernstein LIL}
\author{Francesco Orabona\\
{\tt\small francesco@orabona.com}
}

\date{January 22, 2018}

\begin{document}
\maketitle

\begin{abstract}
This is a verbatim copy of a technical report I wrote in 2017-2018 to obtain the law of the iterated logarithm using the guarantee on the wealth of an online betting strategy.
\end{abstract}

\section{LIL Bound through a Betting Strategy}

First we need a technical lemma.
\begin{lemma}
\label{lemma:l1}
Let $f(\eta)=-\log(1-|\eta|)-|\eta|$, for $-1 \leq \eta\leq 1 $. Then 
\[
f^*(x)=\max_\eta \eta x - f(x) = \max_{-1\leq \eta\leq 1} \eta x - f(x) = |x|-\log(|x|+1)\geq \frac{x^2}{2(|x|+1)}~.
\]
\end{lemma}
\begin{proof}
By definition,
\[
f^*(x)=\max_{\eta} \ \eta x + \log(1-|\eta|)+|\eta|
\]
We equal the derivative wrt $\eta$ to zero to find the minimizer
\[
x-\frac{\sign(\eta^*)}{1-|\eta^*|}+\sign(\eta^*)=0
\]
From which we have that
\[
\eta^*=\frac{x}{|x|+1}~.
\]
Notice that $|\eta^*|\leq 1$, hence $\max_{\eta} \ \eta x -f(\eta) = \max_{-1\leq \eta \leq 1} \ \eta x -f(\eta)$.
Hence
\[
f^*(x)= \frac{x^2}{|x|+1} - \log(|x|+1)+\frac{|x|}{|x|+1} = |x|-\log(|x|+1)
\]
The second lower bound is proved comparing the derivatives.
\end{proof}

\begin{lemma}
Let $\Psi(x)=x-\log(1+x)$, then $\Psi^{-1}(y)=-W_{-1}(-\exp(-y-1))-1$, where $W$ is the Lambert function.
Moreover, the following upper bounds holds
\begin{itemize}
\item $\Psi^{-1}(y)\leq  y+ \log(1+y+\sqrt{2 y}) \leq 2 y +\sqrt{2y}$
\end{itemize}
\end{lemma}
\begin{proof}
For any $x>0$, we have
\begin{align}
\Psi^{-1}(y)
&=-W_{-1}(-e^{-y-1})-1 \\
&=y+\log(-W_{-1}(-e^{-y-1})) \\
&\leq y+ \log(1+y+\sqrt{2 y}),
\end{align}
where the second equality and the first inequality comes from \cite{Chatzigeorgiou13}.
\end{proof}

Define the wealth after $T$ rounds of a betting strategy that bets $w_t$ on the outcome $g_t \in[-1,1]$ and starts with \$1 as
\[
Wealth_T=1+\sum_{t=1}^T w_t g_t
\]

\begin{theorem}
\label{thm:wealth_lower_bound}
Let $|g_i|\leq 1$, then there exists a strategy that guarantees for all $\alpha \in (0,1)$ 
\[
1+\sum_{i=1}^t w_i g_i 
\geq \exp\left(\alpha \Psi\left(\frac{|\sum_{i=1}^t g_i|}{\sum_{i=1}^t g_i^2}\right) \sum_{i=1}^t g_i^2 - \log\frac{\left(\log(\gamma)+\log\frac{1}{\alpha} +\log\left(1+ \frac{\sum_{i=1}^t g_i^2}{|\sum_{i=1}^t g_i|}\right)\right)^2}{0.5 \log(\gamma)\log\frac{1}{\alpha}}\right),
\]
where $\Psi(x)= x- \log(x+1)$
\end{theorem}
\begin{proof}
Define $D(Q,P)=\E_{h\sim Q} [\log \frac{Q(h)}{P(h)}]$.
Set $\alpha \in (0,1)$.

Define the following betting strategy: start with \$1 and bet 
\[
w_t=\E_{\beta\sim P}\left[\beta \prod_{i=1}^{t-1} (1+\beta g_i)\right]
\]
where $P$ is a distribution on $[-1,1]$, symmetric around the origin.
By induction, the wealth has the following expression
\begin{align*}
Wealth_{t}
&= Wealth_{t-1}+w_t g_t\\
&= \E_{\beta\sim P}\left[\prod_{i=1}^{t-1} (1+\beta g_i)\right] + w_t g_t \\
&= \E_{\beta\sim P}\left[(1+\beta g_t)\prod_{i=1}^{t-1} (1+\beta g_i)\right]\\
&=\E_{\beta\sim P}\left[\prod_{i=1}^{t} (1+\beta g_i)\right],
\end{align*}
and $Wealth_0=1$.
In particular, choose $P$ equal to $\frac{\log \gamma}{2}\frac{1}{|\beta| \log(|\beta|/\gamma)^2}$ in $[-1,1]$, where $\gamma>1$ is a free parameter.
We have that this probability distribution has the minimum in $\beta=\gamma/e^2$.


For a given $t$, define $\hat{\beta}_t=\frac{\sum_{i=1}^t g_i}{|\sum_{i=1}^t g_i| + \sum_{i=1}^t g_i^2}$.
Without loss of generality, assume that $\hat{\beta}_t\geq0$.
Choose $Q_t$ equal to P in $[\alpha \hat{\beta}_t,\hat{\beta}_t]$ and renormalized, so that the KL divergence is
\begin{align*}
D(Q_t,P) 
&= \log \frac{1}{\int_{\alpha|\hat{\beta}_t|}^{|\hat{\beta}_t|} P(\beta) d\beta}
\end{align*}

We now use the change of measure lemma, to have
\begin{align*}
\E_{\beta\sim P}\left[\prod_{i=1}^{t} (1+\beta g_i)\right] 
&\geq \exp\left(\E_{\beta\sim Q_t}\left[\sum_{i=1}^{t} \log(1+\beta g_i)\right]-D(Q_t,P)\right) \\
&= \exp\left(\E_{\beta\sim Q_t}\left[\sum_{i=1}^{t} \log(1+\beta g_i)\right]+ \log \int_{\alpha|\hat{\beta}_t|}^{|\hat{\beta}_t|} P(\beta) d\beta\right)
\end{align*}

We have that $\int_{a}^b \frac{\log \gamma}{2}\frac{1}{|\beta| \log(|\beta|/\gamma)^2} d\beta=\frac{0.5 \log \gamma}{\log(\gamma)-\log(b)}-\frac{0.5 \log \gamma}{\log(\gamma)-\log(a)}=\frac{0.5 \log \gamma \log\frac{b}{a}}{(\log(\gamma)-\log(b))(\log(\gamma)-\log(a))}$. If $a=\alpha b$, then $\frac{0.5 \log \gamma \log\frac{1}{\alpha}}{(\log(\gamma)-\log(b))(\log(\gamma)-\log(\alpha b))}\geq \frac{0.5 \log \gamma \log\frac{1}{\alpha}}{(\log(\gamma)-\log(\alpha b))^2}$.

Notice that $\frac{\log(1+\beta x)-\beta x}{\beta^2 x^2} \geq \frac{\log(1-|\beta|)+|\beta|}{\beta^2}$ for $|x|\leq 1$, hence
$\log(1+\beta x)\geq \beta x + x^2 [\log(1-|\beta|)+|\beta|]$, hence
\begin{align*}
\E_{\beta\sim Q_t}\left[\sum_{i=1}^{t} \log(1+\beta g_i)\right]  
&\geq \E_{\beta\sim Q_t} \left[\beta \sum_{i=1}^t g_i + [\log(1-|\beta|)+|\beta|]\sum_{i=1}^t g_i^2\right]~.
\end{align*}
Now, from Lemma~\ref{lemma:l1}, we use the fact that the expression in the expectation in rhs is concave wrt to $\beta$ and the max is attained in $\hat{\beta}_t$ to have
\begin{align*}
\E_{\beta\sim Q_t} \left[\beta \sum_{i=1}^t g_i + [\log(1-|\beta|)+|\beta|]\sum_{i=1}^t g_i^2\right]
&\geq \E_{\beta\sim Q_t}\left[\alpha \hat{\beta}_t \sum_{i=1}^t g_i + [\log(1-|\alpha\hat{\beta}_t|)+|\alpha\hat{\beta}_t|]\sum_{i=1}^t g_i^2\right]\\
&= \alpha \hat{\beta}_t \sum_{i=1}^t g_i + [\log(1-|\alpha\hat{\beta}_t|)+|\alpha\hat{\beta}_t|]\sum_{i=1}^t g_i^2\\
&\geq \alpha (\hat{\beta}_t \sum_{i=1}^t g_i + [\log(1-|\hat{\beta}_t|)+|\hat{\beta}_t|]\sum_{i=1}^t g_i^2)\\
&= \alpha \sum_{i=1}^t g_i^2 (\hat{\beta}_t \frac{\sum_{i=1}^t g_i}{\sum_{i=1}^t g_i^2} + [\log(1-|\hat{\beta}_t|)+|\hat{\beta}_t|])\\
&= \alpha \Psi\left(\frac{|\sum_{i=1}^t g_i|}{\sum_{i=1}^t g_i^2}\right) \sum_{i=1}^t g_i^2\\
\end{align*}
where in the second inequality we used the fact that $\frac{\log(1-x|\beta|)+x|\beta|}{x}\geq \log(1-|\beta|)+|\beta|$ because from the first derivative $\frac{\log(1-x|\beta|)+x|\beta|}{x}$ is decreasing for $x \in (0,1]$, and in the last equality we used Lemma~\ref{lemma:l1}

Putting all together we have
\begin{align*}
&1+\sum_{i=1}^t g_i w_i = \E_{\beta\sim P}\left[\prod_{i=1}^{t} (1+\beta g_i)\right] \\
&\geq \exp\left(\alpha \Psi\left(\frac{|\sum_{i=1}^t g_i|}{\sum_{i=1}^t g_i^2}\right) \sum_{i=1}^t g_i^2 - \log\frac{(\log(\gamma)-\log(\alpha |\hat{\beta}_t|))^2}{0.5 \log(\gamma)\log\frac{1}{\alpha}}\right)\\
&= \exp\left(\alpha \Psi\left(\frac{|\sum_{i=1}^t g_i|}{\sum_{i=1}^t g_i^2}\right) \sum_{i=1}^t g_i^2 - \log\frac{\left(\log(\gamma)+\log\frac{1}{\alpha|\hat{\beta}_t|}\right)^2}{0.5\log(\gamma)\log\frac{1}{\alpha}}\right)~. \qedhere
\end{align*}
\end{proof}

\begin{theorem}
Let $g_1, \cdots, g_t$ a martingale difference sequence, with $|g_t|\leq 1$.
For all $\alpha \in (0,1)$, we have that with probability at least $1-\delta$ uniformly for all $t$ 
\begin{align*}
\left|\sum_{i=1}^t g_i\right| 
&\leq \max\left(S_t \, \Psi^{-1}\left(\frac{1}{\alpha S_t}\log \frac{A_t}{\delta}\right), \sqrt{\frac{2}{\alpha}S_t}\right)\\
&\leq \max\left(\frac{1}{\alpha}\log \frac{A_t}{\delta}  + S_t \log\left(1+\frac{1}{\alpha S_t}\log \frac{A_t}{\delta}+\sqrt{\frac{2}{\alpha S_t}\log \frac{A_t}{\delta}}\right), \sqrt{\frac{2}{\alpha}S_t}\right)\\
&\leq \frac{2}{\alpha}\log \frac{A_t}{\delta} + \sqrt{\frac{2 S_t}{\alpha}\log \frac{A_t}{\delta}},
\end{align*}
where $S_t=\sum_{i=1}^t g_i^2$ and $A_t=\frac{\left(\log(\gamma)+\log\frac{1}{\alpha} +\log\left(1+ \sqrt{\frac{\alpha}{2}S_t}\right)\right)^2}{0.5\log\gamma\log\frac{1}{\alpha}}$
\end{theorem}
\begin{proof}
Consider a given $t \in [1, T]$.

For any betting strategy on $g_1, \cdots, g_t$, we have 
\begin{align*}
\E[Wealth_{t}|g_1, \cdots, g_{t-1}]
= \E[Wealth_{t-1}+w_t g_t|g_1, \cdots, g_{t-1}]
= Wealth_{t-1}~.
\end{align*}
Hence, if the betting strategy guarantees non-negative wealth, we have by Doob’s inequality, that
\[
P\left(\sup_t Wealth_t\geq\frac{1}{\delta}\right)\leq E[Wealth_0]\delta=\delta~.
\]
Using Theorem~\ref{thm:wealth_lower_bound}, this implies
\[
P\left(\sup_t \alpha \Psi\left(\frac{|\sum_{i=1}^t g_i|}{\sum_{i=1}^t g_i^2}\right) \sum_{i=1}^t g_i^2 - \log\frac{\left(\log(\gamma)+\log\frac{1}{\alpha} +\log\left(1+ \frac{\sum_{i=1}^t g_i^2}{|\sum_{i=1}^t g_i|}\right)\right)^2}{0.5\log(\gamma)\log\frac{1}{\alpha}}\geq \log\frac{1}{\delta}\right)
\leq \delta,
\]
that is, uniformly over $t$, with probability at least $1-\delta$ we have that 
\[
\alpha \Psi\left(\frac{|\sum_{i=1}^t g_i|}{\sum_{i=1}^t g_i^2}\right) \sum_{i=1}^t g_i^2 
\leq \log\frac{1}{\delta}\frac{\left(\log(\gamma)+\log\frac{1}{\alpha} +\log\left(1+ \frac{\sum_{i=1}^t g_i^2}{|\sum_{i=1}^t g_i|}\right)\right)^2}{0.5 \log(\gamma)\log\frac{1}{\alpha}}~.
\]
For each $t$, we have two possibilities: $|\sum_{i=1}^t g_i| < \sqrt{\frac{2}{\alpha}\sum_{t=1}^t g_i^2}$ or $|\sum_{i=1}^t g_i| \geq \sqrt{\frac{2}{\alpha}\sum_{t=1}^t g_i^2}$. In the first case, the statement is implied deterministically. In the second case, we can upper bound the term $\log\left(1+ \frac{\sum_{i=1}^t g_i^2}{|\sum_{i=1}^t g_i|}\right)$ with $\log\left(1+ \sqrt{\frac{\alpha}{2}\sum_{i=1}^t g_i^2}\right)$ and solve the inequality for $|\sum_{i=1}^t g_i|$.
\end{proof}

\section{History of this Technical Report}
All the above, apart from this section and the abstract, are a verbatim copy of a technical report I wrote on January 22, 2018, to obtain the law of iterated logarithm from the guaranteed wealth of an online betting algorithm. To preserve its authenticity, I did not change anything, apart from fixing a typo in the title. Given that I shared this report with several people from 2017 on, many can confirm its veracity.

The core result was obtained during the workshop\footnote{\url{https://www.lorentzcenter.nl/theoretical-foundations-for-learning-from-easy-data.html}} on ``Theoretical Foundations for Learning from Easy Data'' at the Lorentz Center, 7--11 November 2016, and immediately shared with some of the participants there. The key idea was to merge my previous work on online betting algorithms with that of \citet{RakhlinS17}, presented at the workshop. However, I did not obtain the law of the iterated logarithm, but rather a time-uniform concentration with a $\log t$-dependence. Later, after reading the blog post\footnote{\url{https://blog.wouterkoolen.info/QnD_LIL/post.html}} of Wouter Koolen, dated 2017-04-26, on a simple proof of the law of iterated logarithm, I wrote this note.

For one year, I did not know what to do with it, until in 2019 we added a small section based on this result in \citet{JunO19}, where we also extended it to unbounded random vectors in Banach spaces. Later, we used the same method again to obtain the optimal constant of the law of the iterated logarithm from the regret of the universal portfolio algorithm in \citet{OrabonaJ21}.

Given the importance that these results have now in the field, I decided to put them on arXiv to precisely delineate the history of these ideas. Also, as far as I know, the exact proof I used here is still unpublished.

\section*{Acknowledgements}
I thank Ashok Cutkosky and Aaditya Ramdas for pointing out small mistakes on a previous version of this report on Nov 30, 2017, and Jan 22, 2018, respectively.

\bibliographystyle{plainnat}
\bibliography{../learning}

\end{document}

%% file: symbol.tex
\newcommand{\field}[1]{\mathbb{#1}}

\newcommand{\E}{\field{E}}

\newtheorem{lemma}{Lemma}
\newtheorem{theorem}{Theorem}

\newcommand{\sign}{{\rm sign}}

%% file: main.bbl
\begin{thebibliography}{4}
\providecommand{\natexlab}[1]{#1}
\providecommand{\url}[1]{\texttt{#1}}
\expandafter\ifx\csname urlstyle\endcsname\relax
  \providecommand{\doi}[1]{doi: #1}\else
  \providecommand{\doi}{doi: \begingroup \urlstyle{rm}\Url}\fi

\bibitem[Chatzigeorgiou(2013)]{Chatzigeorgiou13}
Ioannis Chatzigeorgiou.
\newblock Bounds on the lambert function and their application to the outage
  analysis of user cooperation.
\newblock \emph{IEEE Communications Letters}, 17\penalty0 (8):\penalty0
  1505--1508, 2013.

\bibitem[Jun and Orabona(2019)]{JunO19}
K.-S. Jun and F.~Orabona.
\newblock Parameter-free online convex optimization with sub-exponential noise.
\newblock In \emph{Proc. of the Conference on Learning Theory (COLT)}, 2019.

\bibitem[Orabona and Jun(2024)]{OrabonaJ21}
F.~Orabona and K.-S. Jun.
\newblock Tight concentrations and confidence sequences from the regret of
  universal portfolio.
\newblock \emph{{IEEE} Trans. Inf. Theory}, 70\penalty0 (1):\penalty0 436--455,
  2024.
\newblock URL \url{https://arxiv.org/abs/2110.14099}.

\bibitem[Rakhlin and Sridharan(2017)]{RakhlinS17}
A.~Rakhlin and K.~Sridharan.
\newblock On equivalence of martingale tail bounds and deterministic regret
  inequalities.
\newblock In \emph{Proc. of the Conference On Learning Theory (COLT)}, pages
  1704--1722, 2017.

\end{thebibliography}
